\documentclass[10pt,twocolumn]{article}
\usepackage{iftex}
\ifPDFTeX
  \usepackage[utf8]{inputenc}
  \usepackage[T1]{fontenc}
  \usepackage{tgtermes}
  \newcommand{\arbp}[1]{}
\else
  \usepackage{fontspec}
  \IfFontExistsTF{[amiri-regular.ttf]}
    {\newfontfamily\arabicfont[Script=Arabic]{[amiri-regular.ttf]}%
     \newcommand{\arbp}[1]{ ({\arabicfont #1})}}
    {\IfFontExistsTF{[NotoNaskhArabic-Regular.ttf]}
      {\newfontfamily\arabicfont[Script=Arabic]{[NotoNaskhArabic-Regular.ttf]}%
       \newcommand{\arbp}[1]{ ({\arabicfont #1})}}
      {\newcommand{\arbp}[1]{}}}
\fi
\usepackage{microtype}
\usepackage[top=1in,bottom=1in,left=0.75in,right=0.75in]{geometry}
\usepackage{amsmath,amsthm,amssymb}
\usepackage{booktabs}
\usepackage{multirow}
\usepackage{graphicx}
\usepackage{subcaption}
\usepackage{xcolor}
\usepackage{hyperref}
\usepackage{cleveref}
\usepackage{algorithm}
\usepackage{algpseudocode}
\usepackage{enumitem}
\usepackage{url}
\usepackage{balance}
\usepackage{array}
\usepackage{placeins}   
\usepackage{flafter}    
\usepackage[numbers]{natbib} 

\newtheorem{proposition}{Proposition}
\newtheorem{corollary}{Corollary}[proposition]
\newtheorem{remark}{Remark}
\theoremstyle{definition}
\newtheorem{definition}{Definition}
\definecolor{pipelineblue}{RGB}{52,100,178}
\definecolor{todored}{RGB}{200,30,30}

\hypersetup{
  pdftitle={RefLAM: A Reference-Grounded Line Annotation Pipeline for
            Historical Arabic Manuscripts},
  pdfauthor={Mohamed Guechaoui, Mohamed Diaa Zellagui, Souleyman Chaib, Sahraoui Dhelim},
  colorlinks=true,
  linkcolor=pipelineblue,
  citecolor=pipelineblue,
  urlcolor=pipelineblue
}
\title{\textbf{RefLAM: A Reference-Grounded Line Annotation Pipeline\\
  for Historical Arabic Manuscripts}}
\author{%
  Mohamed Guechaoui$^1$ \and Mohamed Diaa Zellagui$^1$
  \and Souleyman Chaib$^1$ \and Sahraoui Dhelim$^1$
  \\[2pt]
  $^1$Higher School of Computer Science (ESI-SBA), Sidi Bel Abbes, Algeria
}
\date{}
\begin{document}
\maketitle
\begin{abstract}
Existing approaches to building line-level Arabic handwritten-text-recognition
(HTR) training data either rely on fully manual annotation, which does not
scale, or on automatic OCR-to-reference alignment methods that have not been
extended to multi-script, two-zone (main-plus-margin) historical manuscript
layouts with a provable correctness guarantee. We present \textbf{RefLAM}
(\textbf{Ref}erence-grounded \textbf{L}ine \textbf{A}nnotation for
\textbf{M}anuscripts), a pipeline that converts manuscript page images and
pre-existing clean transcriptions into validated, line-level ground truth
without sacrificing human oversight. RefLAM couples a deep-learning
page-segmentation model with a multimodal large language model (MLLM) for
structured OCR and a diacritic-agnostic fuzzy alignment engine that grounds
each OCR line in a contiguous span of the reference text, assigning a
character-level confidence score in $[0,100]$.
RefLAM's efficiency rests on a simple guarantee: under our alignment metric,
a perfect score is provably equivalent to character-for-character identity
of the normalised strings (the \emph{Confidence-100 rule}), and a manual
audit of every perfect-scoring line in the released corpus found no
counterexample. Because a reviewer can trust a perfect score, most lines are
confirmed at a glance rather than retyped, and annotation turns from a
uniformly slow process into a triaged one in which human attention
concentrates where the alignment is uncertain. Across 7 fully page-validated
books we measured a 75$\times$ throughput gain over manual annotation
(3,000 vs.\ 40 lines/hr); applying the same guarantee to a further 7 books,
we retained 16,533 confidence-100 main-text lines for release within one
week, while sub-100 lines from those books were excluded from this release
rather than manually corrected (\cref{sec:validation}).
Using RefLAM, we construct and release AraMS-28k, comprising 14 historical
Arabic manuscript books, 3,043 pages, and 27,971 main-text and 629
margin-line annotations with bounding boxes, layout labels, and insertion
anchors for 191 of the margin entries (30.4\%). Finally, we finetune
Muharaf-pretrained baselines (including HATFormer) on AraMS-28k and report
character error rate (CER) results, confirming the corpus's practical
utility for downstream HTR training.
\end{abstract}
\section{Introduction}
\label{sec:intro}
The Arab-Islamic manuscript tradition spans more than a millennium and covers
medicine, astronomy, mathematics, philosophy, and jurisprudence. A large fraction of this
corpus has never been rendered machine-readable.  Photographic digitisation is
largely solved; the unsolved problem is transcription at scale.  Handwritten-text
recognition (HTR) models are data-hungry, yet line-level labelled corpora for
Arabic manuscripts remain scarce.
Producing line-level OCR training data conventionally requires three expensive
manual stages: page segmentation, transcription, and quality verification.  For
historical Arabic manuscripts, four well-documented difficulties compound this
cost: high script variability (Naskh, Ruq\textquoteleft ah, Thuluth, Maghrebi),
degraded scan quality, non-linear two-zone reading order (main body plus
margins), and diacritisation inconsistency.  We measured the cumulative manual
annotation rate directly: a trained annotator completes roughly 40 lines per
hour, including line detection, transcription, and bounding-box drawing — a
rate at which a single 600-page volume consumes the better part of a year.
We address this bottleneck with \textbf{RefLAM}, which exploits two resources
that, for many classical Arabic texts, already exist independently of manuscript
images: (1)~a vision-capable MLLM that produces structured, layout-tagged OCR
from a single zero-shot call; and (2)~a clean ground-truth transcription—typically
a fully diacritised scholarly edition.  Neither resource alone solves the problem:
MLLM output is fluent but unreliable; the reference text is reliable but not
localised to any page image.  RefLAM fuses them via diacritic-agnostic fuzzy
alignment, attaching a confidence score to every line and triaging human review
accordingly.
RefLAM is best read as an instance of \emph{weak (distant) supervision}
\cite{smith2023automatic}: a noisy labelling
source (MLLM OCR) is reconciled against a reliable but unlocalised one (the
reference transcription) until the labels are both accurate and grounded in
the page; \cref{sec:related} places the pipeline in that literature.
The alignment stage has one property we lean on throughout: a maximal
similarity score is not a heuristic signal but a provable guarantee of
character-for-character normalised-string identity (\cref{prop:conf100}).
We noticed the pattern before we proved it — during page-level review of the
earliest books, no perfect-scoring line was ever found to be wrong — and the
proof then explained the observation. This \emph{confidence-100 rule} is what
makes the fast review path safe; it is an enabling property of the alignment
stage, not a contribution that competes with the pipeline itself.
Although we instantiate and validate RefLAM on Arabic manuscripts, the
pipeline places no Arabic-specific requirement beyond the normalisation
operator of \cref{def:norm}; in principle it should transfer to any
historical script for which clean digital transcriptions exist.
\paragraph{Contributions.}
\begin{enumerate}[leftmargin=*,nosep]
  \item \textbf{RefLAM} (\cref{sec:pipeline}): a five-stage reference-grounded
    annotation pipeline (segmentation~$\rightarrow$ MLLM OCR~$\rightarrow$
    normalisation~$\rightarrow$ page anchoring~$\rightarrow$ diacritic-agnostic
    fuzzy line alignment, \cref{sec:alignment}) that achieves a measured
    75$\times$ speed-up over manual annotation while preserving
    full human review at every confidence level, underpinned by a provable
    alignment guarantee — the \emph{confidence-100 rule}
    (\cref{sec:conf100}) — that makes this speed-up safe rather than
    heuristic.
  \item \textbf{AraMS-28k} (\cref{sec:dataset}): the dataset produced by
    RefLAM, presented here as evidence that the pipeline generalises across
    three hand-copied scripts (Naskh, Ruq\textquoteleft ah, Maghrebi), one
    lithographed volume, and two annotation depths — 14 books, 3,043 pages,
    28,600 line annotations, including the first publicly described
    margin/insertion-anchor annotation at this scale. The dataset itself is
    documented in full in a companion paper~\cite{arams2026}.
  \item \textbf{Baseline HTR results} (\cref{sec:baseline}): finetuning
    results for Kraken and HATFormer on AraMS-28k, reported as downstream
    validation that the released corpus supports HTR training, rather than
    as a modelling contribution in its own right.
\end{enumerate}
\FloatBarrier
\section{Related Work}
\label{sec:related}
\subsection{Arabic Manuscript Datasets}
\label{sec:related:datasets}

Public line-level datasets of genuine historical handwritten Arabic
manuscripts remain scarce because they rely on labor-intensive manual
transcription and verification.  RASM2018 \cite{clausner2018rasm}
provides $\approx$120 pages with manually produced line-level
transcriptions and coarse region labels, but no main/margin distinction.
RASAM \cite{vidalgorene2021rasam} covers the Maghrebi family
($\approx$300 pages, 7,540 lines) and annotates margin regions
(\emph{marginalia}, \emph{catchwords}) through manual segmentation
and transcription.  Muharaf \cite{saeed2024muharaf} is the largest
existing corpus overall ($\approx$36,311 lines), of which 24,495 lines
are publicly released; it includes layout tags for paragraph and
floating (margin/footer) regions, but these were produced through
manual annotation workflows without automated reference alignment.
OpenITI MAKHZAN \cite{allen2026openiti} provides line-level
transcriptions across 1,497 pages predominantly manuscripts, with a small printed-Urdu subset — of multilingual Arabic-script
manuscripts (822 Arabic pages), all manually segmented and 
transcribed.  KHATT \cite{mahmoud2014khatt} covers modern handwritten
Arabic only.  No prior work combines an automated pipeline for
producing verified line-level ground truth with explicit main/margin
layout tags and a provable confidence criterion---precisely the
combination RefLAM provides.
\subsection{OCR and Segmentation for Historical Arabic}
Open-source engines such as Kraken \cite{kiessling2019kraken} and the eScriptorium
platform \cite{kiessling2019escriptorium} are widely used.  The recent HATFormer
\cite{chan2025hatformer} is a Transformer-based recogniser targeting historical
handwritten Arabic, requiring large labelled corpora that remain scarce.  For
line segmentation, U-Net-based \cite{mechi2019unet,mechi2021twostep} and deep
learning \cite{neche2019arabic} approaches demonstrate strong performance.
To bootstrap our annotation pipeline, we adopt the publicly available Kraken
segmentation model trained on the Muharaf corpus \cite{bors2024seg}, which
provides polygonal line-segment predictions as a strong initialisation for
our workflow; the human verification protocol built around it is described
in \cref{sec:segmentation}.
\subsection{Reference-Grounded OCR Alignment}
OCR-to-reference alignment as a labelling mechanism is itself a form of weak
(distant) supervision: a noisy but abundant source (OCR) is reconciled
against a reliable but unlocalised source (a reference transcription) to
produce localised labels without full manual annotation. This principle has
been studied extensively for handwritten text recognition~\cite{smith2023automatic}
and in related document-analysis settings~\cite{toselli2021digital}.

Most directly related to our setting is the ACDC framework of
Smith et al.~\cite{smith2023automatic}, which bootstraps line-level HTR
training data for Arabic-script manuscripts by aligning noisy Kraken HTR
output against clean digital editions via an HMM-based collation model.
ACDC iteratively retrains the HTR model on lines selected by empirical
match-rate and gap thresholds, achieving a 19.6\% absolute character-accuracy
improvement without any manual transcription. However, ACDC does not
annotate layout structure (e.g., main versus margin zones), relies on
heuristic quality criteria rather than a provable correctness guarantee,
and operates as a fully automatic pipeline without human review of
individual alignments.

Do et al.~\cite{do2025reference} localise noisy OCR paragraphs within
e-books via Levenshtein-based fuzzy matching for classical Vietnamese books.
RefLAM pursues the same principle but (i)~operates at \emph{line} rather
than paragraph granularity; (ii)~targets \emph{handwritten} multi-script
two-zone Arabic pages; (iii)~adds an explicit main/margin layout tag to
every line; and (iv)~\emph{formally characterises} the confidence
threshold at which the correspondence is provably exact.

Vision-capable MLLMs such as Gemini \cite{team2023gemini} demonstrate
strong OCR performance on complex layouts
\cite{greif2025multimodal}. Known MLLM OCR failure
modes --- hallucinated characters or words, repeated spans, and silent
omission of illegible text --- are exactly the patterns RefLAM's
alignment stage is designed to catch: in RefLAM, MLLM output is treated
as a noisy hypothesis, and each of these failure modes depresses the
similarity score against the reference and routes the line to detailed
human review (\cref{sec:erroranalysis}).
\FloatBarrier
\section{The RefLAM Pipeline}
\label{sec:pipeline}
\subsection{Overview and Design Principles}
\label{sec:overview}
RefLAM is built under three non-negotiable constraints.
\textbf{Human oversight at all times}: no annotation enters the dataset without
human review at any confidence level.
\textbf{Substantial throughput improvement}: the system must make the majority
of individual review decisions dramatically faster than full manual transcription.
\textbf{Auditability}: every line record is traceable to the reference span that
produced it, and the build is reproducible given fixed configuration and cached
MLLM responses.
Each page traverses five stages (see \cref{fig:pipeline}):
(1)~line segmentation producing bounding boxes and layout labels;
(2)~MLLM structured OCR yielding a layout-tagged transcription hypothesis;
(3)~diacritic-agnostic normalisation of both OCR and reference text;
(4)~page-anchor detection locating the page in the book-level reference;
(5)~greedy line-level fuzzy alignment with confidence scores feeding
human-triaged review.
\begin{figure*}[tbp]
  \centering
  \includegraphics[width=0.78\textwidth]{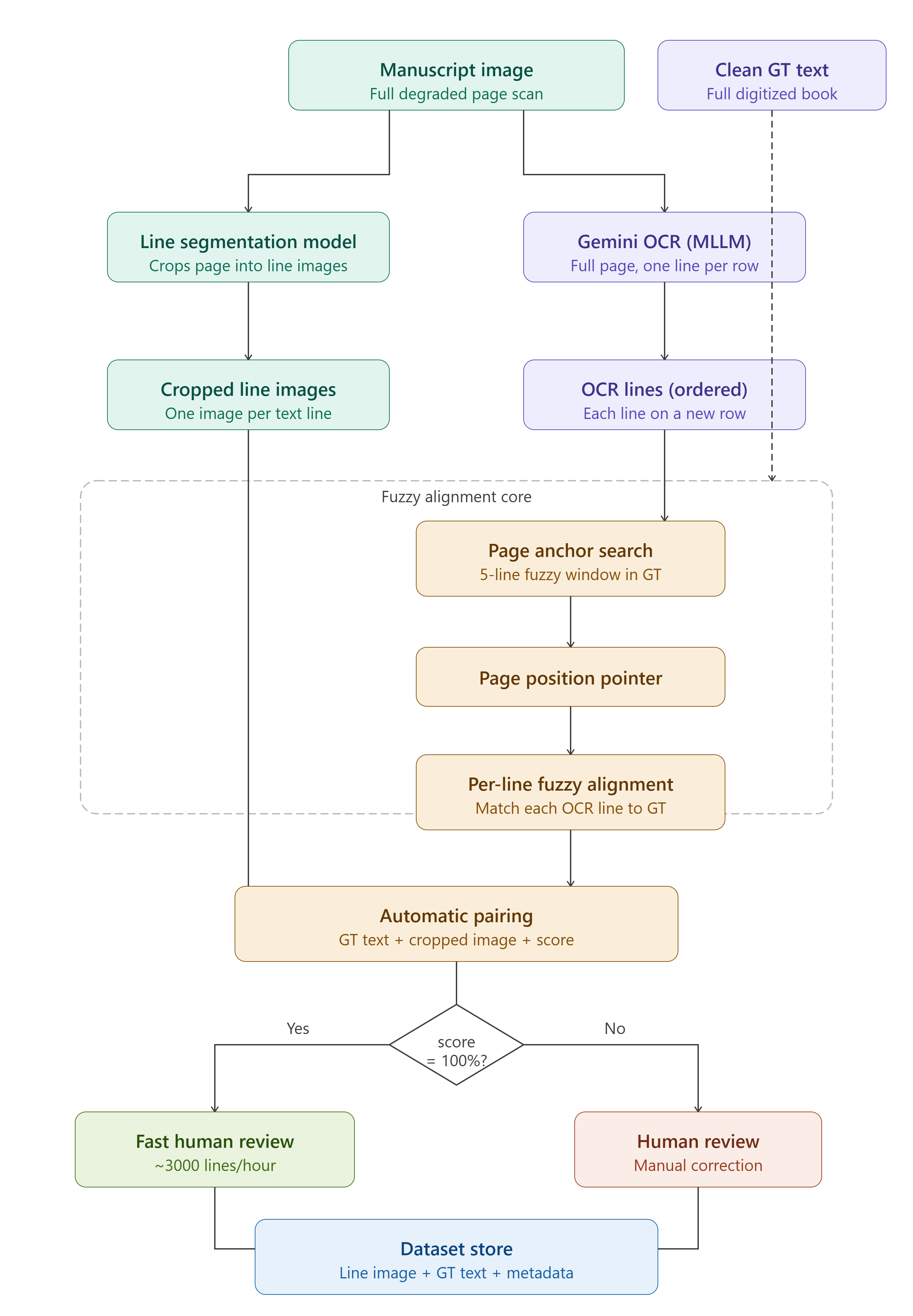}
  \caption{%
    \textbf{The RefLAM pipeline.}
  }
  \label{fig:pipeline}
\end{figure*}
\subsection{Line Segmentation}
\label{sec:segmentation}
To initialise line segmentation, we use the publicly available Kraken
segmentation model (trained on the Muharaf corpus) from Zenodo
\cite{bors2024seg}. This model outputs polygonal contours for each detected
line. Since polygon correction is more involved than rectangle adjustment,
our pipeline converts these polygons to axis-aligned bounding boxes during
preprocessing, while retaining the original polygons for pixel-level mask
generation when needed.
A human reviewer then verifies that no line has been missed, incorrectly
merged, or assigned to the wrong sequential index—issues that can arise due
to severe degradation or overlapping text. Additionally, the fuzzy alignment
stage occasionally reveals segmentation mismatches (e.g., a line split across
two predictions, or a prediction corresponding to the wrong line index),
requiring further human intervention. Approximately 2\% of lines require
manual adjustment, primarily to correct misaligned polygons or resolve
margin--main text confusion. Despite these corrections, the reviewer does not
draw boxes from scratch; they validate, adjust, and correct the
model-provided initialisations.
Line regions are stored in two formats: the majority of lines use
\emph{boundary polygons} (ordered vertex sequences stored as arrays of
$[x, y]$ pairs), which faithfully follow the contour of each text line, and
approximately 2\% of lines use axis-aligned \emph{bounding boxes}
($[x, y, w, h]$).
\begin{figure*}[tbp]
  \centering
  \begin{subfigure}[t]{0.31\textwidth}
    \includegraphics[width=\linewidth]{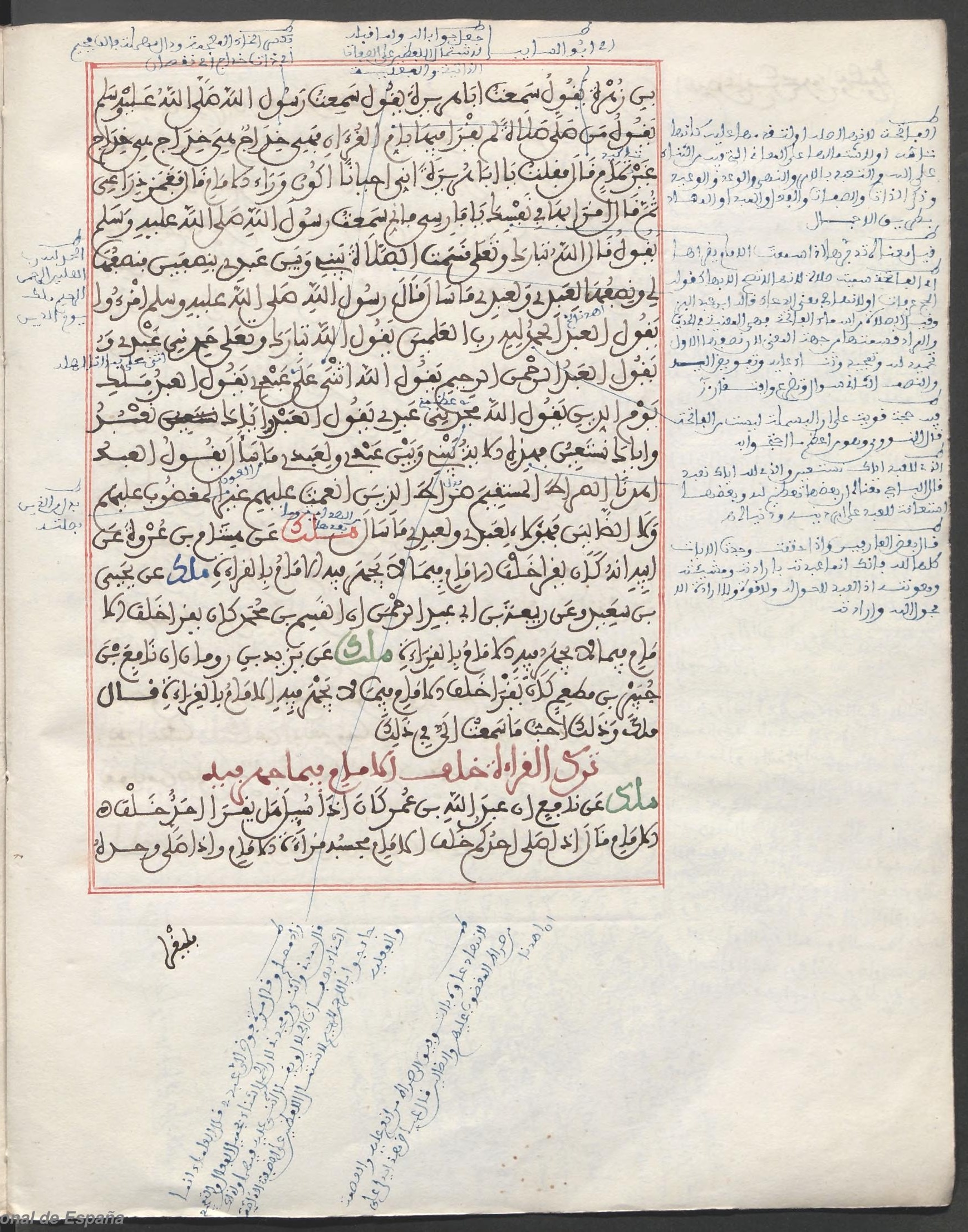}
    \caption{A representative page from \texttt{book\_19} (Maghrebi script).
      }
    \label{fig:sample_book19}
  \end{subfigure}
  \hfill
  \begin{subfigure}[t]{0.31\textwidth}
    \includegraphics[width=\linewidth]{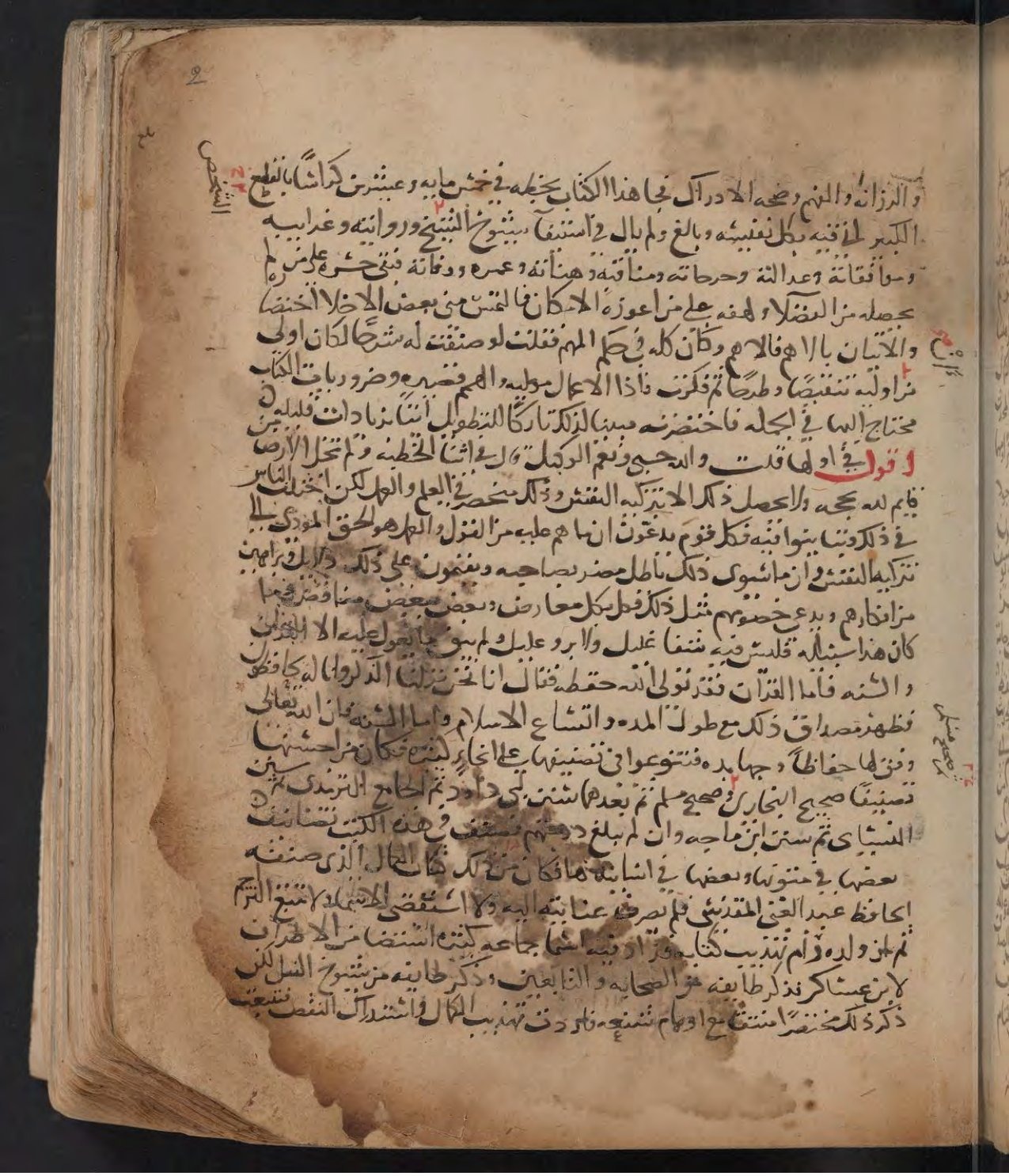}
    \caption{A page from \texttt{book\_09} (Naskh script) }
    \label{fig:sample_book09}
  \end{subfigure}
  \hfill
  \begin{subfigure}[t]{0.31\textwidth}
    \includegraphics[width=\linewidth]{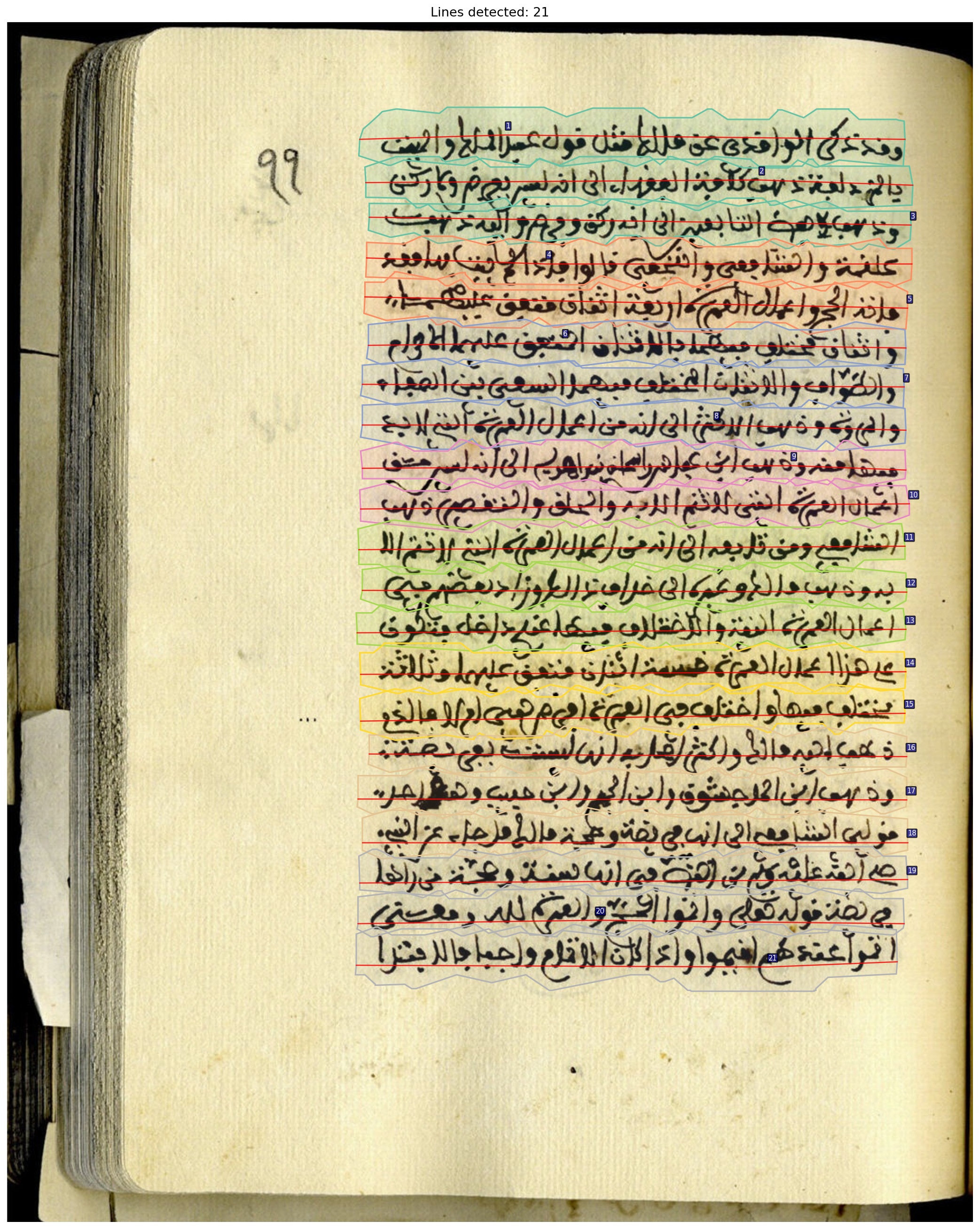}
    \caption{  A page from \texttt{book\_03} segmented }
    \label{fig:sample_segmented}
  \end{subfigure}
  \caption{%
    \textbf{AraMS-28k manuscript samples.}
  }
  \label{fig:samples}
\end{figure*}
\subsection{MLLM Structured OCR}
\label{sec:mllm}
Each page image is submitted \emph{whole} to Google Gemini
(\texttt{gemini-3-flash-preview}) \cite{gemini3flash2026} so that the model can exploit global layout
cues to separate main-body from marginal text.  The model is instructed to emit
each readable text line on its own line; to tag marginal text with an
Arabic-language margin marker, leaving main text untagged; and to return
\emph{only} Arabic text with no commentary.  The full production prompt (v2.3)
is reproduced in \cref{app:prompt}.
The MLLM output is treated throughout as a \emph{noisy label, never as ground
truth}.  Hallucinated content fails alignment and receives a low confidence
score, routing it to detailed review.  The alignment-and-confidence mechanism
therefore doubles as a \emph{hallucination detector}: the reference text is
the anchor of truth.
\subsection{Diacritic-Agnostic Normalisation}
\label{sec:normalisation}
Because manuscript OCR is typically undiacritised while the reference is fully
vocalised, direct string comparison fails.  We normalise both sides identically
before any comparison.
\begin{definition}[Canonical normalisation]
\label{def:norm}
The operator $\mathrm{norm}(\cdot)$ maps a raw string to a canonical string by:
(1)~stripping all Arabic harakat\arbp{ـَ ـُ ـِ ـً ـٌ ـٍ ـّ ـْ}
(U+064B--U+065F) and the superscript alef\arbp{ـٰ} (U+0670);
(2)~removing the kashida\arbp{ـ} (U+0640);
(3)~mapping alef variants\arbp{أ إ آ ٱ} to a single base alef\arbp{ا};
(4)~mapping alef-maq\d{s}\=ura\arbp{ى} to y\=a'\arbp{ي} and
t\=a'-marb\=u\d{t}a\arbp{ة} to h\=a'\arbp{ه};
(5)~removing non-Arabic, non-whitespace, non-digit characters;
(6)~collapsing whitespace and trimming.
\end{definition}
Both raw fields (\texttt{gt\_raw} and \texttt{gemini\_raw}) are preserved in
the dataset, enabling future diacritisation-restoration research.
\subsection{Page-Level Anchor Detection}
\label{sec:anchor}
Before line alignment, each page is located within the book-level reference—the
\emph{page-anchor problem}.  A page signature $P$ is built from the first five
normalised main OCR lines.  A fixed-width window of reference lines is scored
against $P$ at every offset via exhaustive linear scan (\cref{alg:anchor}).
The partial-ratio variant $C_{\text{partial}}$ (\cref{eq:cpartial}) is used
since the page signature may be a substring of a larger reference region.
\begin{algorithm}[tbp]
\caption{Page-anchor detection (exhaustive scan with early exit)}
\label{alg:anchor}
\begin{algorithmic}[1]
\Require Page signature $P$; GT lines $g_{1:N}$; window $w$; lookback $\delta$
\Ensure Anchor offset $s^*$, score $C^*$
\If{$P$ is empty} \Return $\mathit{ptr}$ \EndIf
\State $C^* \leftarrow 0$; $s^* \leftarrow \mathit{ptr}$
\For{$s \leftarrow 0$ \textbf{to} $N-1$}
  \State $W \leftarrow \text{join of normalised } g_s, \ldots, g_{\min(s+w,N)-1}$
  \State $c \leftarrow C_{\text{partial}}(P, W)$
  \If{$c > C^*$} $C^* \leftarrow c$; $s^* \leftarrow s$ \EndIf
  \If{$C^* \geq 95$} \textbf{break} \EndIf
\EndFor
\If{$C^* \geq 50$}
  \Return $\max(0, s^* - \delta)$, $C^*$
\Else\ \Return \textsc{Fail}
\EndIf
\end{algorithmic}
\end{algorithm}
\subsection{Line-Level Fuzzy Alignment}
\label{sec:alignment}
Given a confirmed page anchor, each OCR line is aligned to a contiguous span of
the flat reference word array in reading order (\cref{alg:align}).  The
character-level similarity for ordinary line-to-window matching is the indel
(LCS-based) ratio\cite{hyyro2004bitparallel}:
\begin{equation}
  C(o_i, g_{j:k}) = \frac{2 \cdot \mathrm{LCS}(A, B)}{|A| + |B|} \times 100,
  \label{eq:conf}
\end{equation}
where $A = \mathrm{norm}(o_i)$ and $B = \mathrm{join}(\mathrm{norm}(g_j),
\ldots, \mathrm{norm}(g_k))$, implemented as \texttt{rapidfuzz.fuzz.ratio}\cite{rapidfuzz}.
For page-anchor detection and margin-line matching we use the asymmetric
partial-ratio:
\begin{equation}
  C_{\text{partial}}(A, B) = \max_{B' \sqsubseteq B,\, |B'|=|A|}
    \frac{2 \cdot \mathrm{LCS}(A, B')}{2|A|} \times 100.
  \label{eq:cpartial}
\end{equation}
The alignment is \emph{greedy and single-pass}: each line is matched against
the best window near the current pointer; the pointer advances past the matched
span; the next line is searched from there.
\begin{algorithm}[tbp]
\caption{Greedy windowed line alignment}
\label{alg:align}
\begin{algorithmic}[1]
\Require Lines $o_{1:M}$; GT word array length $N$; pointer $\mathit{ptr}$;
  lookback $L$; window $F$; tolerance $\tau$;
  $\theta_{\text{main}}{=}72$, $\theta_{\text{margin}}{=}55$
\Ensure Per-line spans $(j_i, k_i)$ and confidences $C_i$
\For{$i \leftarrow 1$ \textbf{to} $M$}
  \State $a \leftarrow \mathrm{norm}(o_i)$; skip if $|a| < \textsc{MinOcrChars}$
  \State $n \leftarrow$ word count of $a$
  \If{$o_i$ is a main line}
    \State Range $[\mathit{ptr}{-}L,\, \mathit{ptr}{+}F]$; windows $[\max(2,n{-}\tau),\, n{+}\tau]$
    \State Score with $C$ \eqref{eq:conf}; early-exit if $\geq 98$
    \State Threshold $\leftarrow \theta_{\text{main}}$
  \Else
    \State Range $[\max(\mathit{ptr}_{\text{anchor}},\, e_{\text{ref}} - 30),\;
      \min(e_{\text{page}} + 30, N)]$, where $e_{\text{ref}}$ is the word-end
      of the nearest matched main line (or $\mathit{ptr}_{\text{anchor}}$ if none)
    \State Score with $C_{\text{partial}}$ \eqref{eq:cpartial}; early-exit if $\geq 98$
    \State Threshold $\leftarrow \theta_{\text{margin}}$
  \EndIf
  \If{best score $\geq$ threshold}
    \State Record span and score; if main, advance $\mathit{ptr}$ past span
  \Else
    \State Record miss (confidence 0); $\mathit{ptr}$ unchanged
  \EndIf
\EndFor
\end{algorithmic}
\end{algorithm}
\paragraph{Parameter selection and sensitivity.}
All window and threshold parameters ($L$, $F$, $\tau$,
$\theta_{\text{main}}{=}72$, $\theta_{\text{margin}}{=}55$,
\textsc{MinOcrChars}) were tuned manually on the first books processed
during an initial end-to-end iteration of the pipeline, then frozen for the
remainder of construction. Two design choices were validated by ablation
during development: removing the page-anchor stage (\cref{sec:anchor})
increases alignment errors by a factor of 3--5, and enforcing the monotonic
pointer advance improves line-order coherence over unconstrained per-line
matching. Because every line is subsequently human-reviewed
(\cref{sec:humanreview}), these parameters affect how efficiently lines are
routed for review, not whether errors enter the release: a sub-optimal
threshold sends more lines to detailed review; it does not admit mistakes.
\subsection{The Confidence-100 Rule}
\label{sec:conf100}
\begin{proposition}[Confidence-100 rule]
\label{prop:conf100}
Let $A = \mathrm{norm}(o_i)$ and $B = \mathrm{norm}(g_{j:k})$.  If
$C(o_i, g_{j:k}) = 100$ under \cref{eq:conf}, then $A = B$
character-for-character.
\end{proposition}
\begin{proof}
$C = 100$ iff $2\,\mathrm{LCS}(A,B) = |A| + |B|$.  Since
$\mathrm{LCS}(A,B) \leq \min(|A|,|B|)$, equality forces $|A|=|B|=:n$
and $\mathrm{LCS}(A,B)=n$; a common subsequence of length $n$ between two
length-$n$ sequences can only be the full sequence, hence $A=B$.
\end{proof}
Consequently, $C=100$ certifies character-for-character identity in normalised
space, regardless of any diacritisation or letter-variant differences between
raw manuscript and raw reference text—precisely the differences that
\cref{def:norm} absorbs.
\begin{corollary}[Margin lines]
For margin lines scored with $C_{\text{partial}}$, a score of 100 implies $A$
is identical to some contiguous window $B'$ of the searched region; the matched
text is correct, though the offset may not be unique for short repeated
formulae.
\end{corollary}
\begin{remark}
\cref{prop:conf100} guarantees normalised-string identity.  It does not certify
that the diacritisation in \texttt{gt\_raw} is scholarly-correct, that the
bounding box is pixel-perfect, or that the matched span is unique in the corpus.
\end{remark}
\noindent\textbf{Empirical verification.}  The rule was conjectured before it
was proved: during full-page review of the first books, neither reviewer ever
encountered an incorrect confidence-100 line. To test the guarantee at corpus
scale, we then manually audited every confidence-100 line across all 14 books
in a confirmation tool showing the page image, OCR text, and matched
reference span side by side. The audit found zero errors.
\noindent\textbf{Important distinction.}  RefLAM does \emph{not} auto-accept
confidence-100 lines.  The rule is a property a reviewer \emph{relies on during
inspection} to move quickly—but does not bypass inspection itself.
\subsection{Human Review}
\label{sec:humanreview}
Every line, at every confidence level, is reviewed by a human before entering
the released dataset. Confidence-100 lines are confirmed through a \emph{rapid
visual-comparison path} (page image, segmented box, and matched reference span
side by side, confirmed in $\approx$1--2 seconds). Sub-100 lines are reviewed
in detail, with the reviewer free to accept, correct, or reject the match.
To support this process, we developed a dedicated single-page browser-based
review application. The tool presents three panels simultaneously: a line list
with confidence scores and layout tags, an editing panel for corrections, and
a zoomable page-image panel. Keyboard shortcuts support sustained review flow
across large volumes of lines. The tool was used for all human review in this
work and is released alongside the dataset.
\FloatBarrier
\section{The Resulting Corpus: AraMS-28k}
\label{sec:dataset}
RefLAM's end-to-end output is AraMS-28k, summarised in
Table~\ref{tab:dataset_summary}: 14 historical Arabic books — thirteen
hand-copied manuscripts spanning three script traditions (Naskh,
Ruq\textquoteleft ah, Maghrebi) and one lithographed printed edition —
comprising 3,043 pages, 27,971 main-text lines, and 629 margin lines
(sample pages in Figure~\ref{fig:samples}). For every line the pipeline
emits a layout label (\texttt{main}/\texttt{margin}), geometry
(\cref{sec:segmentation}), the aligned reference transcription with its
confidence score, a coarse orientation estimate, and — for margin lines with an unambiguous attachment point in the
main text, 191 of 629 ($\approx\!30\%$) --- a human-assigned
\emph{insertion anchor} recovering the page's non-linear reading order.
The remainder are retained with a null anchor: some are structurally
non-referential (catchwords, signatures, section explanations), while
others lack a clear attachment point; in neither case is a guess forced. Construction proceeded in two review
phases that differ in depth (\cref{sec:validation}): 548 fully
page-validated pages and 2,495 line-validated pages, from which only
confidence-100 main-text lines are released.
The corpus is split at the book level (train: 9 books, 19,739 lines;
validation: 2 books, 1,486 lines; test: 3 fully page-validated books, 6,746
lines, one per hand-copied script). Complete documentation — the full
annotation schema, layout statistics, comparison with prior corpora,
datasheet, and release formats — is provided in the companion dataset
paper~\cite{arams2026}; here the corpus serves as evidence of what the
pipeline produces at scale, and \cref{sec:analysis} analyses the pipeline
statistics behind it.
\begin{table}[htbp]
  \centering
  \caption{AraMS-28k at a glance. Full documentation in~\cite{arams2026}.}
  \label{tab:dataset_summary}
  \small
  \renewcommand{\arraystretch}{1.15}
  \begin{tabular}{@{}p{4.2cm}p{2.8cm}@{}}
    \toprule
    \textbf{Property} & \textbf{Value} \\
    \midrule
    Books & 14 \\
    Total pages & 3,043 \\
    \quad Fully page-validated & 548 \\
    \quad Line-validated & 2,495 \\
    Main-text lines & 27,971 \\
    Margin lines & 629 \\
    Scripts & Naskh, Ruq\textquoteleft ah, Maghrebi \\
            & (+1 lithographed vol.) \\
    Anchored margin lines & 191 / 629 ($\approx\!30\%$) \\
    Train / val / test split & 9 / 2 / 3 books \\
    License & CC BY-NC-SA 4.0 \\
    \bottomrule
  \end{tabular}
\end{table}
\FloatBarrier
\section{Analysis}
\label{sec:analysis}
\subsection{Confidence Score Distribution}
Figure~\ref{fig:confidence} shows the distribution of alignment confidence
scores by book, binned in ten-point increments, for the page-validated
subset. Two patterns stand out. First, the distribution is strongly bimodal:
for most books the majority of lines land either in the 90--100 bin or below
50, with comparatively few lines in the intermediate 60--89 range. This is
consistent with the failure modes in \cref{sec:erroranalysis} being largely
binary in effect—either the MLLM output matches the reference near-exactly,
or a specific disruption (hallucination, degradation, margin overlap) causes
a sharp drop in similarity rather than a graceful partial match. Second, the
proportion of mass in the top bin varies substantially by book (cf.\ the
C=100 column of Table~\ref{tab:perbook}), tracking script regularity and scan
quality more than any other single factor: \texttt{book\_06} (Naskh, clean
scan) sits at 99.9\% confidence-100, while \texttt{book\_03} (Maghrebi,
degraded scan) sits at 8.4\%.
\begin{table}[htbp]
  \centering
  \caption{%
    Per-book alignment agreement. ``C=100'' is the share of reviewed main
    lines achieving a perfect alignment score against the reference
    (\cref{prop:conf100}), plotted in Figure~\ref{fig:confidence}.
    \textbf{For LV books, C=100 is 100\% \emph{by construction}}: only
    confidence-100 lines were retained for release from these books
    (\cref{sec:validation}). \texttt{book\_10} is a lithographed printed
    edition rather than a hand-copied manuscript. Full per-book dataset
    statistics appear in~\cite{arams2026}.
  }
  \label{tab:perbook}
  \small
  \renewcommand{\arraystretch}{1.1}
  \begin{tabular}{@{}lllrrr@{}}
    \toprule
    Book & Type & Script & Pages & Lines & C=100 \\
    \midrule
    \texttt{book\_03} & PV & Maghrebi     & 172 & 3,661 & 8.4\% \\
    \texttt{book\_05} & PV & Ruq\textquoteleft ah & 95 & 2,028 & 28.0\% \\
    \texttt{book\_06} & PV & Naskh        &  41 &   874 & 99.9\% \\
    \texttt{book\_09} & PV & Naskh        &  46 & 1,057 & 9.8\% \\
    \texttt{book\_10} & PV & Lithograph   &  13 &   674 & 14.1\% \\
    \texttt{book\_11} & PV & Naskh        &  30 &   612 & 12.1\% \\
    \texttt{book\_27} & PV & Naskh        & 151 & 2,532 & 11.2\% \\
    \midrule
    \textit{PV subtotal} & & & 548 & 11,438 & 21.0\% \\
    \midrule
    \texttt{book\_12} & LV & Naskh   & 144 & 2,015 & 100\% \\
    \texttt{book\_16} & LV & Naskh   & 392 & 2,276 & 100\% \\
    \texttt{book\_17} & LV & Naskh   & 584 & 3,731 & 100\% \\
    \texttt{book\_19} & LV & Maghrebi& 122 &   666 & 100\% \\
    \texttt{book\_20} & LV & Maghrebi& 312 & 2,163 & 100\% \\
    \texttt{book\_21} & LV & Ruq\textquoteleft ah & 496 & 3,282 & 100\% \\
    \texttt{book\_24} & LV & Maghrebi& 445 & 2,400 & 100\% \\
    \midrule
    \textit{LV subtotal} & & & 2,495 & 16,533 & 100\% \\
    \midrule
    \textbf{Total} & & Mixed & \textbf{3,043} & \textbf{27,971} & \\
    \bottomrule
  \end{tabular}
\end{table}
\begin{figure}[htbp]
  \centering
  \includegraphics[width=\columnwidth]{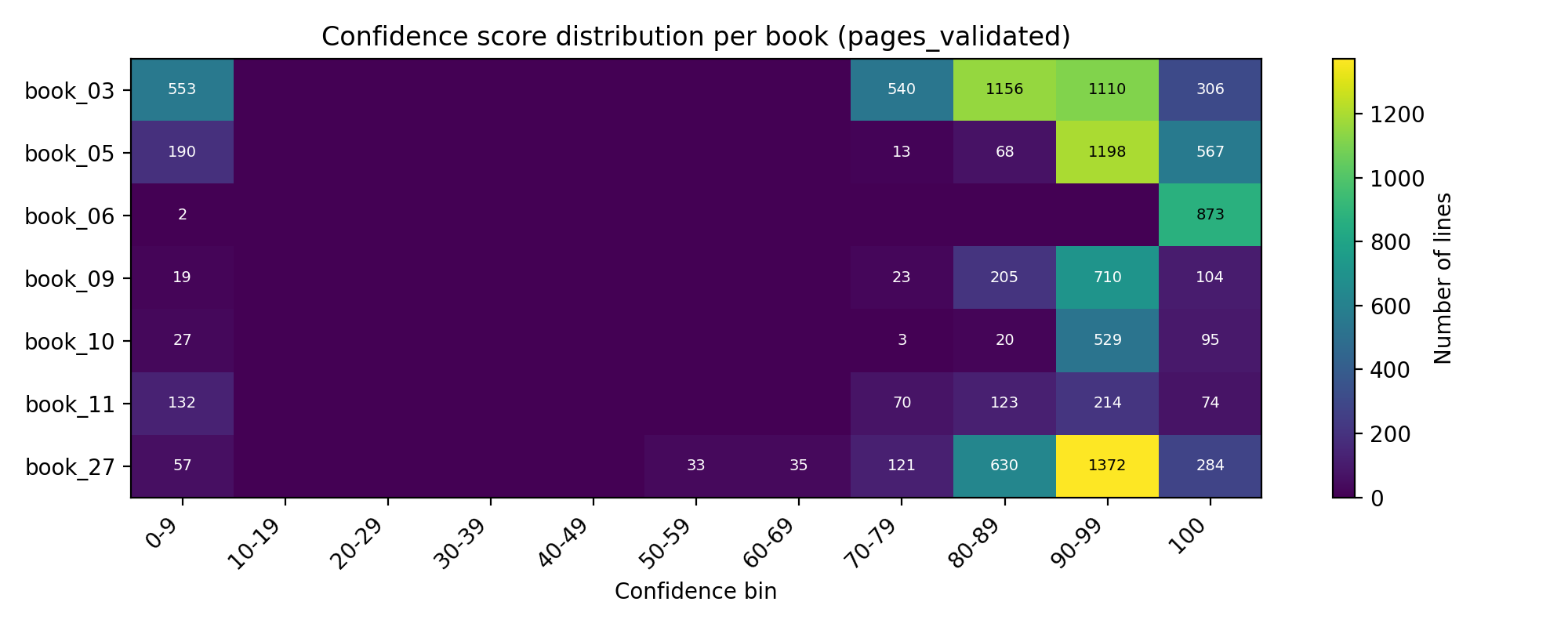}
  \caption{%
    \textbf{Confidence-score distribution per book (page-validated books).}
  }
  \label{fig:confidence}
\end{figure}
\subsection{Error Analysis: Low-Confidence Lines}
\label{sec:erroranalysis}
Lines in the intermediate confidence range (roughly 60--99) require detailed
review.  We identify three representative failure modes:
\noindent\textbf{Scan degradation.}  Faded ink or bleed-through causes the
MLLM to omit or hallucinate characters, producing a low fuzzy score against
the correct reference span.
\noindent\textbf{Unusual script forms.}  Rare ligatures or calligraphic variants
not captured by the orthographic-variant-merging step of \cref{def:norm} produce
spurious mismatches.
\noindent\textbf{Margin--main overlap.}  Margin lines physically overlapping
main-text lines cause the MLLM to merge them into a single multi-line OCR
output, which cannot align cleanly to a single reference span.
All such cases are retained with their sub-100 confidence scores and review
metadata intact, making them identifiable for researchers who wish to filter
or study them specifically.
\FloatBarrier
\section{Validation and Throughput}
\label{sec:validation}
Validation followed a two-phase protocol informed by empirical observation
during the annotation process, summarised in Table~\ref{tab:perbook}.
\paragraph{Phase 1 --- Page-Validated (PV) books.}
The seven PV books (\texttt{book\_03}, \texttt{05}, \texttt{06}, \texttt{09},
\texttt{10}, \texttt{11}, \texttt{27}; 548 pages, 11,438 lines) were validated
by two independent reviewers, with every line inspected.  Empirically,
$C{=}100$ lines were error-free across all seven books.
\paragraph{Phase 2 — Line-Validated (LV) books.}
For the seven LV books (\texttt{book\_12}, \texttt{16}, \texttt{17},
\texttt{19}, \texttt{20}, \texttt{21}, \texttt{24}; 2,495 pages, 16,533
lines), we adopted a more selective annotation strategy. Rather than manually
correcting every misaligned or low-confidence line, we retained only
main-text lines that achieved a perfect Confidence-100 score; all sub-100
lines were excluded from this release rather than manually corrected. This
filtering leaves a subset of lines that are provably character-for-character
identical to the reference text under \cref{prop:conf100}. These retained
lines were confirmed through a rapid visual-comparison path—page image,
matched reference span, and OCR text side by side—typically resolved in
$\approx$0.4--2 seconds per line. This stratified approach dramatically
reduced review time while ensuring that every included annotation benefits
from the Confidence-100 guarantee and retains human oversight.
\subsection{Throughput Analysis}
We measured the baseline throughput by timing a trained annotator's full
validation workflow—adjusting bounding boxes pre-generated by a model trained
on Muharaf~\cite{saeed2024muharaf} and manually verifying/correcting
transcriptions—on a representative 50-page ($\approx$2,250-line) subset.
This baseline came out at approximately 40 lines per person-hour. Sub-100
lines were reviewed at essentially that manual rate. For confidence-100
lines the character-level alignment is already perfect, so the annotator
only confirms the segmentation at a glance — about 3,000 lines per hour, a
75$\times$ internal speedup over the manual baseline, measured on the same
PV books under the same conditions.
For context, Muharaf~\cite{saeed2024muharaf} collected 36,311 lines over
roughly twelve months; assuming a single full-time annotator
(160 hrs/month), that is an upper bound of about 19 lines/hr. Our peak
3,000 lines/hr on confidence-100 lines represents a $\sim$158$\times$
speedup over that rate, and the whole 14-book corpus took one calendar month
with two annotators — the LV phase alone about a week.
\begin{figure}[htbp]
  \centering
  \includegraphics[width=\columnwidth]{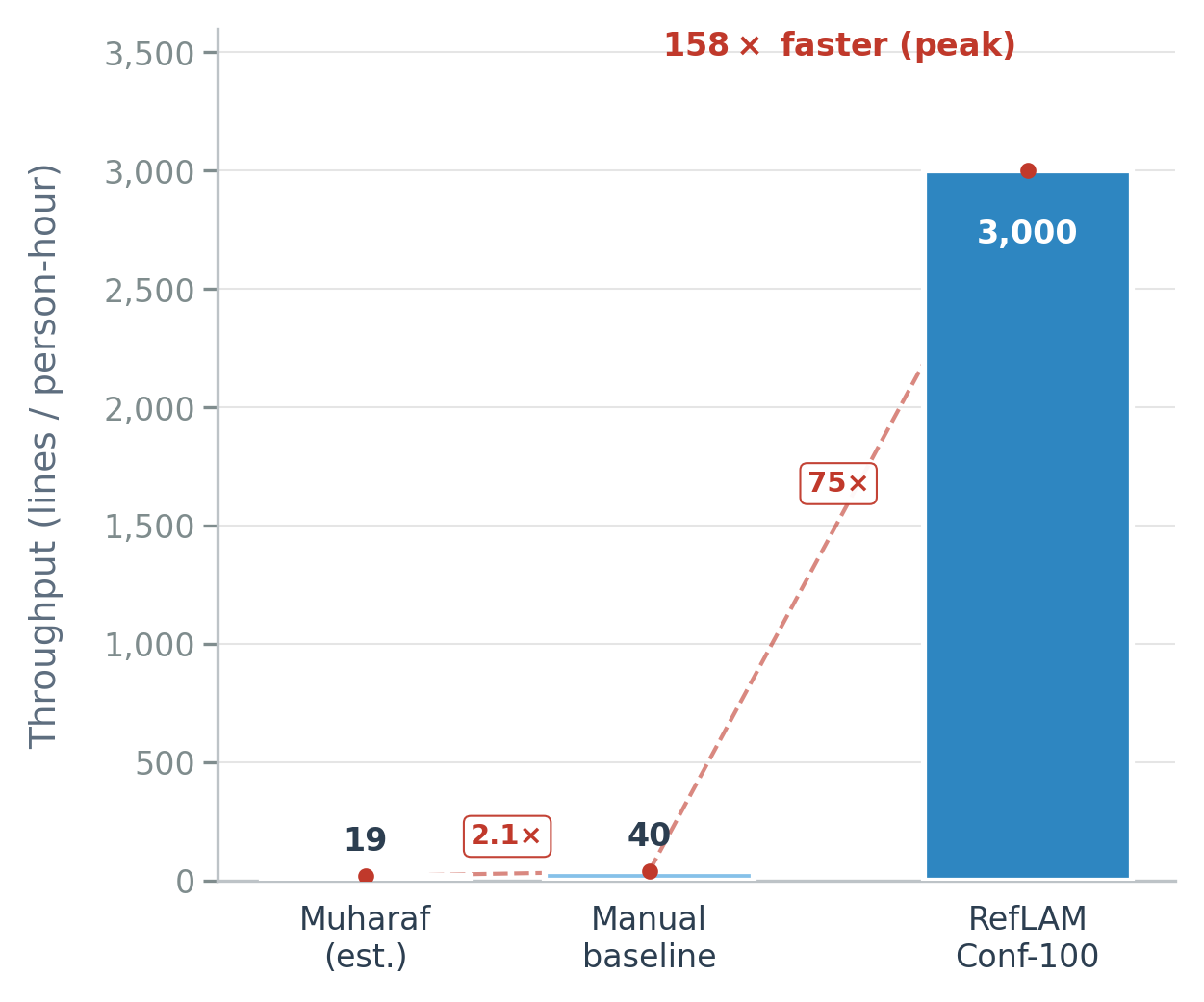}
  \caption{Throughput comparison (lines per person-hour).}
  \label{fig:throughput}
\end{figure}

\subsection{Baseline HTR Results}
\label{sec:baseline}
To demonstrate AraMS-28k's utility for downstream training, we finetune two
Muharaf-pretrained recognition models on the AraMS-28k training split
(9 books, 19,739 lines) and evaluate on the held-out test split (3 books,
6,746 lines): (1) a standard Kraken
recogniser~\cite{kiessling2019kraken,bors2024rec}, and
(2) HATFormer~\cite{chan2025hatformer}. Both models share the same
pretraining corpus (Muharaf) and the same in-domain finetuning data,
isolating the effect of model architecture.
Table~\ref{tab:baseline} reports character error rates (CER) per book and
overall. Under identical finetuning conditions, Kraken outperforms HATFormer
overall (weighted CER: 23.31\% vs.\ 26.74\%), and per-book CER follows the
same ordering for both architectures --- Ruq\textquoteleft ah (11.65\% /
13.26\%), Naskh (22.62\% / 25.37\%), Maghrebi (32.71\% / 37.88\%). This
ordering is not explained by finetuning-data volume: Ruq\textquoteleft ah
has the smallest script-specific training set yet the lowest error, while
Maghrebi has substantially more data yet the highest. Instead, the gradient
tracks proximity to the pretrained Muharaf distribution, whose samples are
predominantly Ruq\textquoteleft ah script~\cite{saeed2024muharaf} --- our
Ruq\textquoteleft ah test book therefore benefits from a close
in-distribution match with the pretraining corpus itself, whereas Maghrebi
is a paleographically distinct tradition largely absent from that corpus's
Levantine-letters composition, so the gap originates in the pretrained
script prior rather than in the amount of AraMS-28k supervision. The full
experimental setup, per-book results table, a diagnostic in-distribution
condition, and the analysis of this cross-script gradient are given in the
dataset paper~\cite{arams2026}; here these baselines serve as downstream
validation that RefLAM's output supports practical HTR training. The finetuned
checkpoints  are released alongside the dataset
(\cref{sec:availability}).
\begin{table}[htbp]
  \centering
  \caption{CER after finetuning Muharaf-pretrained models on AraMS-28k.}
  \label{tab:baseline}
  \small
  \renewcommand{\arraystretch}{1.1}
  \begin{tabular}{@{}l r@{}}
    \toprule
    Model / Test subset & CER (\%) \\
    \midrule
    \multicolumn{2}{c}{\textbf{Kraken}~\cite{bors2024rec}} \\
    \quad \texttt{book\_03} (Maghrebi) & 32.71 \\
    \quad \texttt{book\_05} (Ruq\textquoteleft ah) & 11.65 \\
    \quad \texttt{book\_09} (Naskh) & 22.62 \\
    \quad \textbf{Overall} & \textbf{23.31} \\
    \midrule
    \multicolumn{2}{c}{\textbf{HATFormer}~\cite{chan2025hatformer}} \\
    \quad \texttt{book\_03} (Maghrebi) & 37.88 \\
    \quad \texttt{book\_05} (Ruq\textquoteleft ah) & 13.26 \\
    \quad \texttt{book\_09} (Naskh) & 25.37 \\
    \quad \textbf{Overall} & \textbf{26.74} \\
    \bottomrule
  \end{tabular}
\end{table}
\section{Conclusion}
\label{sec:conclusion}
We presented \textbf{RefLAM}, a pipeline that turns manuscript page images
and pre-existing clean transcriptions into validated, line-level ground
truth by treating MLLM OCR as a noisy hypothesis and the reference text as
the anchor of correctness. At its centre is the confidence-100 rule — a
maximal indel-similarity score forces character-for-character identity of
the normalised strings (\cref{prop:conf100}) — verified across the full
corpus without a single counterexample; it is what allows a 75$\times$
annotation speed-up without surrendering human oversight. Applied end to
end, the pipeline produced AraMS-28k (14 books, 3,043 pages, 28,600 line
annotations with layout labels and insertion anchors) in one calendar month
with two annotators; the corpus itself, its benchmark, and its full
documentation are the subject of the companion dataset
paper~\cite{arams2026}, and the Kraken and HATFormer finetuning experiments
(\cref{sec:baseline}) confirm that the pipeline's output supports practical
downstream HTR training.
Nothing in the design is specific to Arabic beyond the normalisation
operator of \cref{def:norm}: the reference-grounding principle transfers to
any historical script for which clean digital transcriptions exist —
Ottoman Turkish, Persian, Hebrew, and Syriac are plausible candidates.
\paragraph{Limitations.}
RefLAM requires a clean transcription to exist before annotation can begin.
Approximately 70\% of margin lines could not be confidently anchored.  The
line-level aligner is greedy rather than globally optimal.  The MLLM
implementation relies on a proprietary model, mitigated by response caching and
a narrow data contract designed to accept an open-source substitute.
\paragraph{Future work.}
(1)~Replacing the proprietary MLLM with an open-source vision-language model;
(2)~extending RefLAM to Persian and Ottoman Turkish manuscripts;
(3)~training and benchmarking additional HTR architectures on AraMS-28k beyond
the two baselines reported here;
(4)~replacing the coarse rotation estimate with continuous angle regression;
(5)~developing a semi-automated margin-to-main-line linking model.
\paragraph{Release.}
We release the dataset, train/validation/test splits, the RefLAM annotation
pipeline, and the browser-based correction tool under CC~BY-NC-SA~4.0. See
\cref{sec:availability} for access details.
\FloatBarrier

\section*{Data and Code Availability}
\label{sec:availability}

All resources are released under CC~BY-NC-SA~4.0.

\begin{itemize}[leftmargin=*,nosep]
  \item \textbf{AraMS-28k} (Images + JSONL, $\approx$2.05~GB): Zenodo
    \url{https://doi.org/10.5281/zenodo.22095333}; mirror
    \url{https://github.com/ArchaText/AraMS-28k-Dataset}.
  \item \textbf{AraMS-28k-HTR} (PNG crops + .gt.txt, $\approx$2.42~GB):
    HTR-ready line crops automatically derived from AraMS-28k and
    distributed pre-built on Zenodo
    \url{https://doi.org/10.5281/zenodo.21499649}.
  \item \textbf{Pipeline / review tool:}
    \url{https://github.com/ArchaText/Reflam-pipeline}.
\end{itemize}

SHA-256 checksums are supplied in each Zenodo archive for integrity
verification.

\bibliographystyle{unsrtnat}
\bibliography{references}
\FloatBarrier
\appendix
\section{RefLAM MLLM OCR Prompt (v2.3)}
\label{app:prompt}
The production prompt is reproduced verbatim below,
\begin{quote}
\small\ttfamily
You are performing OCR on a handwritten Arabic manuscript page.\\
The page has TWO types of text:\\
1. MAIN TEXT: the regular lines in the center/body of the page.\\
2. MARGIN TEXT: text in the margins (sides, top, bottom),\\
\phantom{xxx}often sideways or squeezed in.\\
\\
Rules:\\
1. Output EXACTLY the Arabic text --- do NOT summarize.\\
2. Do NOT add commentary or explanation.\\
3. For MAIN TEXT: output each physical line as one line,\\
\phantom{xxx}top to bottom, no numbering.\\
4. For MARGIN TEXT: prefix with a masrgin tag only.\\
\phantom{xxx}Example: ``and the writer said'' [margin writing].\\
5. Output all main lines first, then margin lines.\\
6. If no margin text, just output the main lines.\\
7. Do not add any english writing or any explanation with the output.\\
\\
Return ONLY Arabic text. No introduction. No extra text.
\end{quote}
\noindent\textbf{Version history.}
v1.0: Basic OCR, no layout differentiation.
v1.5: Added margin tag; main text untagged to reduce token overhead.
v2.0: Added instruction to include marginal content regardless of size or
orientation.
v2.3~(current): Added strict output constraints to minimise hallucination.

\balance
\end{document}